\documentclass[11pt]{article}

\usepackage[preprint]{acl}
\usepackage{bbm}
\usepackage{capt-of}
\usepackage{times}
\usepackage{latexsym}
\usepackage{amsmath,amssymb,amsthm}
\newtheorem{proposition}{Proposition}
\usepackage[T1]{fontenc}

\usepackage[utf8]{inputenc}
\usepackage{amsthm}
\usepackage{xcolor}
\newtheorem{theorem}{Theorem}
\newtheorem{corollary}{Corollary}
\usepackage{microtype}
\usepackage{wrapfig}

\usepackage{inconsolata}

\usepackage{graphicx}
\usepackage{amsmath}
\usepackage{pifont}
\usepackage{pifont}
\usepackage{booktabs}

\usepackage{algorithm}
\usepackage{algpseudocode}

\usepackage{xcolor}
\usepackage{soul} 
\usepackage{comment}

\usepackage{xcolor}
\usepackage{soul} 
\usepackage{enumitem} 

\newcommand{\tofill}[1]{{\color{red}\textbf{[TO FILL: #1]}}}

\title{\textsc{ForkGRPO}: Efficient Exploration and Fine-Grained Credit Assignment for Diffusion Language Models}

\author{
  Zhaoyan Jia\textsuperscript{1,*}\quad
  Chengwei Zhou\textsuperscript{2,*}\quad
  Gourav Datta\textsuperscript{2}\quad
  Massoud Pedram\textsuperscript{1} \\
  \textsuperscript{1}University of Southern California, Los Angeles, USA \\
  \textsuperscript{2}Case Western Reserve University, Cleveland, USA \\
  {\small\textsuperscript{*}Equal contribution.}
}

\begin{document}
\maketitle
\providecommand{\tofill}[1]{\textbf{[#1]}}
\begin{abstract}
Diffusion large language models (dLLMs) have advanced rapidly, yet strengthening their reasoning through reinforcement learning (RL) remains an open challenge. Any-order generation can suffer from entropy degradation, which narrows the reasoning paths available for exploration. Recent work mitigates this issue by applying standard GRPO with left-to-right generation. However, independently generating every rollout from the input prompt can leave alternative continuations from intermediate states underexplored, while trajectory-level advantages assign uniform credit to all tokens without distinguishing the contributions of individual reasoning segments. We propose \textsc{ForkGRPO}, an RL framework that couples efficient exploration with fine-grained segment-level credit assignment. It samples multiple continuations from high-entropy intermediate states along a reference trajectory, prioritizing textual boundaries and reusing shared prefixes. The continuations provide complementary learning signals: within-state group-relative advantages guide which continuations to reinforce, while Monte Carlo state-value differences indicate which reasoning segments improve expected outcomes. Among the evaluated architecture-preserving RL methods, \textsc{ForkGRPO} achieves the best performance on three of four benchmarks, while improving over the directly comparable JustGRPO by up to 1.2 percentage points and using up to 27.8\% fewer generated tokens. \textsc{ForkGRPO-fast} further achieves a 2.49$\times$ training wall-clock speedup with negligible accuracy loss.

\end{abstract}

\section{Introduction}
\label{sec:introduction}

Diffusion large language models (dLLMs) have advanced rapidly, emerging as a promising alternative to autoregressive (AR) language models with competitive language modeling and downstream capabilities \citep{nie2025llada}. Their iterative denoising process supports parallel token generation and flexible decoding orders, but how to turn these capabilities into stronger reasoning through RL remains an open challenge. Recent work has shown that any-order decoding suffers from \emph{entropy degradation}: uncertain decisions are deferred until other generated tokens have constrained their possible outcomes, narrowing the reasoning paths available for exploration \citep{ni2026justgrpo}. JustGRPO mitigates this issue by enforcing left-to-right generation during RL, requiring the model to sample at uncertain logical forks before subsequent context constrains the available reasoning paths \citep{ni2026justgrpo}. This provides a promising starting point for dLLM reasoning, but leaves an important question: how can training explore the solution space more efficiently and thoroughly, and how can finer-grained credit assignment lead to better reasoning trajectories?

However, standard GRPO still has two limitations. First, \emph{exploration relies on independent rollouts from the input prompt}. Although left-to-right generation preserves alternative reasoning paths, it does not ensure that these possibilities are sufficiently explored. Under a finite sampling budget, rollouts may repeatedly follow common paths, leaving alternative continuations from intermediate states underexplored.
Revisiting these states and sampling multiple continuations could allocate the budget more effectively and improve coverage of the available solution space.

Second, \emph{credit is assigned at the trajectory level}. GRPO computes a relative advantage from the final reward and applies the same advantage to every token in a response \citep{shao2024deepseekmath}. This signal compares complete responses but does not distinguish progress made by individual reasoning segments. A segment that increases the likelihood of eventual success may receive a negative advantage because a later mistake leads to an incorrect answer. Conversely, all segments in a response with a positive advantage are reinforced together, regardless of their individual contributions. Estimating how the expected final reward changes across a segment could provide more fine-grained credit assignment and help the model learn better reasoning trajectories. Together, these limitations motivate sampling alternative continuations from shared intermediate states and deriving segment-level credit from their terminal rewards.


We propose \textsc{ForkGRPO}, an RL framework for dLLMs that combines exploration from intermediate states with segment-level credit assignment.
ForkGRPO samples a group of responses and uses one as the reference trajectory, along which it selects high-entropy fork states while prioritizing textual boundaries.
It samples multiple continuations from each fork state, exploring alternative reasoning paths without generating every response from the beginning.
Their average terminal rewards estimate state values, whose changes provide segment-level advantages along the reference trajectory.
Combined with group-relative advantages among continuations from each fork state, these provide complementary learning signals that guide both progress along a reasoning trajectory and the choice of continuations from its intermediate states. 
We further introduce \textsc{ForkGRPO-Fast} to reduce sampling and
optimization costs. Adaptive rollout adjusts continuation counts based on
pilot reward diversity, while sparse branch replay and state subsampling
reduce gradient computation over collected branches and generation states.

Our methods outperform all compared RL methods for dLLM on mathematical reasoning benchmarks. \textsc{ForkGRPO} outperforms JustGRPO on all four benchmarks by an average of 0.75 percentage points, while using at most 81.3\% of its token-generation budget on GSM8K.
\textsc{ForkGRPO-fast} further achieves a $2.49\times$ training speedup over full \textsc{ForkGRPO} while retaining 90.3\% accuracy on GSM8K. These results highlight the benefits of jointly improving exploration and credit assignment for dLLM reasoning.


\section{Related Work}
\label{sec:related_work}
\vspace{-2mm}

\paragraph{Diffusion language models.}
Diffusion language models formulate text generation as an iterative denoising process that recovers clean sequences from corrupted inputs.
Discrete diffusion methods extend this framework to categorical token spaces \citep{austin2021structured, lou2024discrete}, while masked diffusion models learn to reconstruct masked tokens at varying noise levels \citep{sahoo2024simple}.
Recent models such as LLaDA and Dream scale this approach to large language models, using bidirectional context to support parallel token generation and flexible decoding orders \citep{nie2025llada, ye2025dream7bdiffusionlarge}.
Block Diffusion combines autoregressive dependencies across blocks with diffusion within each block, supporting variable-length generation and efficient inference \citep{arriola2025block}.
These developments motivate studying how RL can further strengthen the reasoning capabilities of dLLMs.

\paragraph{Reinforcement learning for diffusion language models.}
Masked dLLMs generate text through iterative denoising, making the evaluation of generation probabilities a central challenge for RL \citep{nie2025llada, ye2025dream7bdiffusionlarge}. Diffu-GRPO uses a mean-field approximation and single-forward-pass probability estimates \citep{zhao2025d1}, while wd1 reduces reliance on likelihood ratios through a weighted log-likelihood objective \citep{tang2026wd1}. GDPO and ESPO use ELBO-based sequence-level surrogates \citep{rojas2026improving, ou2026principled}, and SPG combines lower and upper likelihood bounds to handle positive and negative advantages \citep{wang2026spg}. DiSPO instead defines a state-conditioned filling objective, sampling alternative completions from cached logits under a factorized policy \citep{oba2026dispo}. The accuracy of likelihood approximations remains an important concern: single-pass estimates can omit dependencies introduced during generation, and token-level ELBO contributions do not generally represent conditional log-probabilities \citep{ou2026principled, pan2026dtreerpo}.
JustGRPO adopts left-to-right generation to align rollout sampling with policy optimization, allowing exact evaluation of next-token probabilities under the sequential rollout policy while retaining bidirectional attention \citep{ni2026justgrpo}.
However, it still samples complete responses independently and assigns the same advantage to all tokens within each response.
\textsc{ForkGRPO} builds on this policy with entropy-guided forking to explore alternative reasoning paths more efficiently and thoroughly, and segment-level credit assignment to provide finer feedback on reasoning progress.

\paragraph{Fine-grained credit assignment and branching.}
For autoregressive LLMs, VinePPO and SPO-chain obtain intermediate credit through Monte Carlo value estimation using auxiliary continuations \citep{kazemnejad2025vineppo, guo2025spo}.
\textsc{ForkGRPO} additionally trains on these continuations through within-state group-relative updates, using the sampling budget for both value estimation and policy learning.
SPO-tree and d-TreeRPO reuse samples through recursive rollout trees \citep{guo2025spo, pan2026dtreerpo}, while TreePO, TreeRL and TreeAdv use uncertainty-guided branching \citep{li2025treepo,hou2025treerl, cao2026treeadv}. \textsc{ForkGRPO} combines entropy-guided branching with Monte Carlo value estimation using nonrecursive groups of full sequential continuations along a reference trajectory.
This design concentrates sampling on uncertain reasoning decisions while reusing existing reasoning.
The resulting segment credit measures changes in expected terminal reward, rather than redistributing trajectory-level advantages as in TreeAdv or relying on one-step greedy value proxies as in EGSPO-SA \citep{kunde2026egspo}. Segment-level credit and within-state comparisons provide complementary feedback on which segments improve expected outcomes and which continuations to reinforce.

\begin{figure*}[t]
    \centering
    \includegraphics[trim={85 195 52 160}, clip, width=1\linewidth]{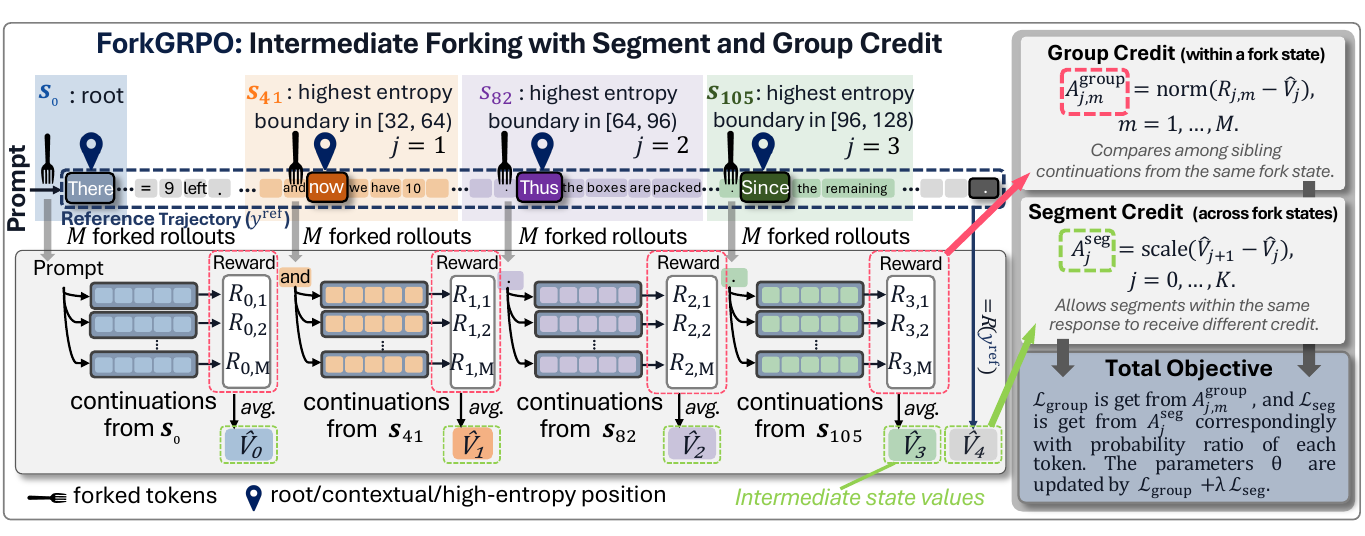}
    \vspace{-8mm}
    \caption{
    \textbf{Overview of \textsc{ForkGRPO}.}
    \textsc{ForkGRPO} samples multiple continuations from high-entropy intermediate states along a reference trajectory and uses their terminal rewards to estimate intermediate state values. These samples provide two complementary credit signals for joint policy optimization: within-state relative credit among sibling continuations and segment-level credit derived from the change in value between successive fork states. 
    }
    \vspace{-4mm}
    \label{fig:overview}
\end{figure*}

\section{Method}
\label{sec:method}
\vspace{-2mm}

In this section, we present \textsc{ForkGRPO}, illustrated in Figure~\ref{fig:overview}. We first describe
how fixed windows, textual boundaries, and predictive entropy guide branching along a reference trajectory.
We then derive within-state group-relative credit and across-state segment credit from the sampled
continuations and combine them in the training objective. Finally, we introduce \textsc{ForkGRPO-Fast} to reduce rollout generation and gradient replay costs.

\subsection{Boundary-Aware Forking of a Reference Trajectory}
\label{sec:method_forking}

\paragraph{Intermediate states and reference trajectory.} We collect trajectories using left-to-right sequential unmasking while retaining the model's bidirectional attention over the partially masked sequence, following the rollout protocol of JustGRPO \citep{ni2026justgrpo} for RL training of masked dLLMs. This ordering requires the model to resolve uncertain reasoning decisions before proceeding to subsequent content, while allowing alternative continuations to be explored from intermediate states. For a completed response $y$ of length $L$ for a given prompt $q$, we define the state immediately before revealing position $i$, where $0 \leq i < L$, as
\vspace{-2mm}
\begin{equation}
    s_i(q,y)
    =
    \bigl(q,y_{<i},[\mathrm{MASK}]^{L-i}\bigr).
    \label{eq:sequential_state}
    \vspace{-2mm}
\end{equation}
Thus, every position along a response defines an intermediate state with a committed prefix and an unresolved masked suffix.

For each prompt, we sample $M$ complete responses from the fixed rollout policy $\pi_{\mathrm{old}}$ and take the first as the reference trajectory $y^{\mathrm{ref}}$. \textsc{ForkGRPO} selects $K$ intermediate states along this trajectory from which to sample alternative continuations.

\paragraph{Windowing.}
To determine where to sample intermediate continuations, we first partition the early portion of the response into \mbox{$K$} non-overlapping windows,
\begin{equation}
\mathcal{W}_j = \{jW, \dots, (j+1)W-1\}, \quad j = 1, \dots, K.
\end{equation}
Windowing prevents a small number of locally uncertain positions from consuming the entire branching budget and provides rollouts at multiple stages of the trajectory.  We focus these windows on earlier portions of the response, where alternative continuations retain a longer horizon over which they can diverge or recover. Consistent with this intuition, later forks empirically exhibit substantially higher rates of zero group advantage (Section~\ref{sec:branch_depth_analysis}).

\paragraph{Boundary-aware forking.}
Within each window, not every token position provides a meaningful intermediate reasoning state. Forking inside a phrase or local expression may primarily perturb lexical realization rather than a completed reasoning decision. We therefore prioritize textual transition points and define
$\mathcal P_j\subseteq\mathcal W_j$ as positions immediately following a newline or sentence-ending punctuation. If no such position exists, we fall back to $\mathcal P_j=\mathcal W_j$. The boundary delimiters are specified in Appendix~\ref{app:fork_config}.

We further rank these fork position candidates by predictive entropy. This choice is motivated by the observation of JustGRPO~\citep{ni2026justgrpo} that high-uncertainty tokens often occur at branching points in reasoning, where different token choices can lead to different downstream solution paths. Consistent with this view, we also observe that branching from high-entropy states can recover correct continuations even when the reference trajectory ultimately fails. For revealing reference token $i$, we compute
$
    H_i^{\mathrm{ref}}=-\sum_{v\in\mathcal V}
    p_i(v)\log p_i(v),
$
where $\mathcal V$ is the token vocabulary,
$v\in\mathcal V$ denotes a candidate token,
and $p_i(v)=\pi_{\mathrm{old}}(v\mid s_i^{\mathrm{ref}})$
is its conditional probability.
The selected fork is
\vspace{-2mm}
\begin{equation}
    a_j
    =
    \arg\max_{i\in\mathcal P_j}
    H_i^{\mathrm{ref}}.
    \label{eq:fork_anchor}
    \vspace{-2mm}
\end{equation}

For sequential masked generation, this uncertainty signal is available directly from the token distribution at each partially masked state and requires no auxiliary model. At each selected state \mbox{$s_{a_j}^{\text{ref}}$}, we keep \mbox{$y_{<a_j}^{\text{ref}}$} fixed and sample \mbox{$M$} independent continuations starting from position \mbox{$a_j$}, so that the selected token itself is resampled. Together with the root state \mbox{$a_0 = 0$}, this produces \mbox{$K+1$} continuation groups. We denote completion \mbox{$m$} from group \mbox{$j$} by \mbox{$y^{(j,m)}$} and its terminal reward by \mbox{$R_{j,m}$}. 

\subsection{Within-State Group-Relative Credit}
\label{sec:method_group}

Each fork state produces a group of continuations that share the same
reference prefix but explore different future trajectories. These sibling
continuations provide a natural within-state comparison: their terminal
rewards reveal which alternatives are preferable when starting from the same
intermediate state.
We compute the group mean and sample variance as
\vspace{-2mm}
\begin{equation}
\begin{aligned}
    \bar R_j
    &=
    \frac{1}{M}\sum_{m=1}^{M}R_{j,m},\\
    \widehat{\sigma}_j^2
    &=
    \frac{1}{M-1}\sum_{m=1}^{M}
    (R_{j,m}-\bar R_j)^2.
\end{aligned}
\vspace{-2mm}
\end{equation}
The corresponding group-relative advantage is
\vspace{-2mm}
\begin{equation}
    A_{j,m}^{\mathrm{group}}
    =
    \frac{R_{j,m}-\bar R_j}
    {\widehat{\sigma}_j+\epsilon_A},
    \label{eq:group_advantage}
    \vspace{-2mm}
\end{equation}
where $\epsilon_A>0$ is a numerical stabilizer.
This advantage is shared by all newly generated
tokens in response $y^{(j,m)}$.
Because all continuations within a group start from the same intermediate
state, this advantage isolates variation among alternative futures rather
than differences in their preceding prefixes. It therefore provides
the credit for learning which continuations should be made
more or less likely.

\subsection{Across-State Segment Credit}
\label{sec:method_segment}

Group-relative credit distinguishes alternative continuations from the same
state, but it does not indicate how the reference trajectory itself changes
the prospect of eventual success as generation proceeds. \textsc{ForkGRPO} obtains
this complementary signal by interpreting the mean reward of each fork group
as a Monte Carlo estimate of intermediate state value.

Specifically, for fork state $j$, we define
\vspace{-2mm}
\begin{equation}
    \widehat V_j = \bar R_j
    =
    \frac{1}{M}\sum_{m=1}^{M}R_{j,m},
    \quad j=0,\ldots,K.
    \label{eq:mc_value}
    \vspace{-2mm}
\end{equation}
Let $a_{K+1}=L$ denote the terminal boundary. For each segment $\mathcal I_j=\{a_j,\ldots,a_{j+1}-1\}$, we measure the change in expected terminal outcome between its two boundaries:
\vspace{-2mm}
\begin{equation}
    \Delta V_j=
    \begin{cases}
        \widehat V_{j+1}-\widehat V_j,
        & j=0,\ldots,K-1,\\
        R(q,y^{\mathrm{ref}})-\widehat V_j,
        & j=K.
    \end{cases}
    \label{eq:segment_advantage}
    \vspace{-1mm}
\end{equation}

A positive $\Delta V_j$ indicates that, after traversing segment
$\mathcal I_j$, subsequent continuations achieve a higher expected terminal
reward, whereas a negative value indicates that the segment decreases the
estimated prospect of success.

This state-to-state comparison addresses a limitation of response-level
credit assignment. A single response-level advantage assigns identical credit
to every segment of a trajectory, even when different segments have different
effects on the expected terminal reward~\citep{kazemnejad2025vineppo}. If the
underlying segment value changes differ, no single shared advantage can
represent all of them accurately, resulting in an irreducible granularity
error. In contrast, \textsc{ForkGRPO} estimates each $\Delta V_j$ separately. Under our
sampling assumptions, the mean squared estimation error of these segment-level
estimates decreases as $O(1/M)$ and, for sufficiently large $M$, becomes lower
than that of a shared response-level advantage. We formalize this result in
Appendix~\ref{app:mc_estimation}
(Theorem~\ref{thm:segment_mc}).

\vspace{-1mm}
\paragraph{Noise-aware shrinkage.} 
The finer granularity, however, introduces Monte Carlo uncertainty because each state value is estimated from only finitely many continuations. Under the same assumptions, the mean squared error of \mbox{$\Delta V_j$} is bounded by \mbox{$u_j = 2\nu^2/M$} for intermediate segments and \mbox{$\nu^2/M$} for the terminal segment, where \mbox{$\nu^2 = (r_{\max}-r_{\min})^2/4$} denotes the global variance upper bound. We therefore use \mbox{$u_j$} as a noise scale and attenuate value changes that are small relative to their estimation uncertainty:
\begin{equation}
\alpha_j = \frac{(\Delta V_j)^2}{(\Delta V_j)^2 + u_j + \epsilon_s}, \quad A_j^{\text{seg}} = \alpha_j \Delta V_j,
\label{eq:segment_noise}
\end{equation}
where \mbox{$\epsilon_s > 0$} stabilizes the denominator. Consequently, segment value changes that are large relative to their uncertainty are largely preserved, whereas weak, potentially noise-dominated estimates are shrunk toward zero. The properties and potential bias of this shrinkage rule are discussed in Appendix~\ref{app:segment_shrinkage}. The resulting \mbox{$A_j^{\text{seg}}$} serves as the segment-level advantage used to optimize the corresponding portion of the reference trajectory.

\subsection{ForkGRPO Objective}
\label{sec:method_objective}

\textsc{ForkGRPO} combines the two complementary credit signals introduced above.
The group-relative advantage $A_{j,m}^{\mathrm{group}}$ compares alternative
continuations sampled from the same intermediate state, whereas the segment
advantage $A_j^{\mathrm{seg}}$ measures the value change between successive
states along the reference trajectory. We optimize both signals using the same
clipped policy objective, but apply them to different parts of the sampled
trajectories.

For token $i$ of continuation $y^{(j,m)}$, we define the policy ratio
\vspace{-2mm}
\begin{equation}
    \rho_i^{(j,m)}
    =
    \frac{
        \pi_\theta(y_i^{(j,m)}\mid s_i^{(j,m)})
    }{
        \pi_{\mathrm{old}}(y_i^{(j,m)}\mid s_i^{(j,m)})
    },
    \label{eq:policy_ratio}
    \vspace{-2mm}
\end{equation}
where $s_i^{(j,m)}=s_i(q,y^{(j,m)})$.
Following the clipped GRPO objective, we write
\vspace{-2mm}
\begin{equation}
    \ell(\rho,A)
    =
    \min\!\left\{
        \rho A,\,
        \operatorname{clip}(\rho,1-\epsilon,1+\epsilon)A
    \right\},
    \label{eq:clipped_surrogate}
    \vspace{-2mm}
\end{equation}
where $\epsilon$ controls the clipping range.

\paragraph{Group objective.}
For each fork group, $A_{j,m}^{\mathrm{group}}$ is applied only to tokens
generated after its fork state; the shared reference prefix is not replayed
with group-relative credit. The resulting objective is
\vspace{-2mm}
\begin{equation}
\begin{aligned}
    \mathcal L_{\mathrm{group}}
    =
    -\frac{1}{N_{\mathrm{tok}}}
    \sum_{j=0}^{K}
    \sum_{m=1}^{M}
    \sum_{i=a_j}^{L-1}
    \ell\!\left(
        \rho_i^{(j,m)},
        A_{j,m}^{\mathrm{group}}
    \right),
\end{aligned}
\label{eq:group_loss}
\vspace{-2mm}
\end{equation}
where
$N_{\mathrm{tok}}=M\sum_{j=0}^{K}(L-a_j)$
is the total number of newly generated tokens.

\paragraph{Segment objective.}
The segment advantage is instead applied to the corresponding portion of the
reference trajectory. Denoting the reference-policy ratio at position $i$ by
$\rho_i^{\mathrm{ref}}$, we optimize
\vspace{-2mm}
\begin{equation}
    \mathcal L_{\mathrm{seg}}
    =
    -\frac{1}{L}
    \sum_{j=0}^{K}
    \sum_{i\in\mathcal I_j}
    \ell\!\left(
        \rho_i^{\mathrm{ref}},
        A_j^{\mathrm{seg}}
    \right).
    \label{eq:segment_loss}
    \vspace{-2mm}
\end{equation}
Thus, different portions of the same reference response receive different
policy updates according to their estimated changes in continuation value.

The final \textsc{ForkGRPO} objective is
\vspace{-2mm}
\begin{equation}
\mathcal{L}_{\text{ForkGRPO}} = \mathcal{L}_{\text{group}} + \lambda \mathcal{L}_{\text{seg}}, \label{eq:forkgrpo_obj}
\vspace{-2mm}
\end{equation}
where \mbox{$\lambda$} controls the contribution of segment-level credit. The reference response participates in both objectives: its root-group advantage provides within-state relative credit, while \mbox{$A_j^{\text{seg}}$} assigns segment-specific credit along the same trajectory. 

\subsection{ForkGRPO-Fast: Adaptive Rollout and Sparse Replay}
\label{sec:forkgrpo_fast}
\textsc{ForkGRPO} uses conditional continuations for two purposes: exploring alternative
outcomes to estimate credit, and replaying sampled trajectories to compute
policy gradients. These two stages need not use the same computational budget.
We therefore introduce \textsc{ForkGRPO-Fast}, which reduces cost at three
levels: adaptive rollout avoids unnecessary continuation sampling, sparse
branch replay reduces the number of continuations revisited with gradients,
and state subsampling further reduces the number of replayed states within
each retained continuation.

\paragraph{Adaptive rollout.}
For each root or fork state, we first sample \mbox{$M_{\text{pilot}} < M$} continuations. If every pilot group is homogeneous and all groups share the same observed reward, both group-relative and segment credit vanish under the pilot samples, so we stop rollout collection for that prompt. Otherwise, only groups with reward variation are expanded to the full size \mbox{$M$}:
\begin{equation}
M_j = \begin{cases} M, & \text{if group } j \text{ reward varies}, \\ M_{\text{pilot}}, & \text{otherwise.} \end{cases}
\end{equation}
All value estimates and group statistics use the samples actually collected. Because group expansion depends on observed pilot outcomes, adaptive rollout introduces a slight selection bias relative to fixed-$M$ sampling (Appendix~\ref{app:fast_adaptive}). However, experimentally this approximation reduces rollout overhead by $1.76\times$ while maintaining task accuracy. 

\paragraph{Sparse branch replay.}
Adaptive rollout reduces exploration cost; once rewards have been collected,
we can separately reduce the cost of gradient replay. For binary-reward tasks,
responses with the same reward within a group also share the same standardized
group-relative advantage. It is therefore unnecessary to replay every branch
from the same reward class. We uniformly retain a small number of
representatives and weight them by their class multiplicity. We consider
\emph{1+1} replay, retaining at most one positive and one negative branch. Conditioned on the
collected rollout set, this multiplicity correction preserves the full group
gradient in expectation. We apply this optimization only to binary-reward
tasks; experiments with non-binary code rewards retain all sampled branches.

\paragraph{State subsampling.}
For each retained branch, we replay only a fraction $\rho$ of the generated
states after the fork position. Random stratified subsampling uses inverse
retention weighting.
The two replay mechanisms reduce local replay cost approximately as
$
    \frac{C_{\mathrm{fast}}}{C_{\mathrm{full}}}
    \approx
    \frac{r_b}{M}\rho,
    \label{eq:fast_cost}
$
where $r_b$ is the number of replayed branches. With $M=8$, \emph{1+1}
replay, and $\rho=0.25$, this retains at most $6.25\%$ of the original
branch--state replay operations.

Overall, \textsc{ForkGRPO-Fast} separates the budget used to discover useful
credit from the computation used to optimize it: adaptive rollout reduces
unnecessary exploration, while branch and state replay sparsify the subsequent
gradient computation.

\section{Experiments}
\label{sec:experiments}
\vspace{-2mm}

\subsection{Experimental Setup}
\label{sec:exp_setup}

We use LLaDA-8B-Instruct~\citep{nie2025llada} as the base model and evaluate
on GSM8K~\citep{cobbe2021gsm8k} and MATH-500~\citep{lightman2023lets} for
mathematical reasoning, and HumanEval~\citep{chen2021evaluating} and
MBPP~\citep{austin2021program} for code generation. We report exact-match
accuracy for mathematics and pass@1 for code, with a fixed response length
of $256$ tokens.
Full reward definitions, decoding settings, \textsc{ForkGRPO}/\textsc{ForkGRPO-Fast} 
configurations, and optimization hyperparameters are provided in
Appendix~\ref{app:experimental_details}.

\subsection{Main Results}
\label{sec:main_results}

Table~\ref{tab:main_results} compares \textsc{ForkGRPO} with prior dLLM reasoning methods at a generation length of 256 tokens. For JustGRPO, we use the LoRA results reported in its official repository,\footnote{\url{https://github.com/LeapLabTHU/JustGRPO}} whose LoRA configuration we also adopt for \textsc{ForkGRPO}. Results for the remaining baselines are taken from their corresponding papers.

\begin{table}[t]
\centering
\fontsize{7.4pt}{6pt}\selectfont
\setlength{\tabcolsep}{0.5pt}
\caption{
\textbf{Main results at generation length 256.}
We report accuracy (\%) on GSM8K and MATH-500 and pass@1 (\%) on HumanEval
and MBPP. 
$\dagger$ indicates architecture modifications;
$\ddagger$ indicates training with substantially more private data.
}
\vspace{-2mm}
\label{tab:main_results}
\begin{tabular}{lcccc}
\toprule
\textbf{Method}
& \textbf{GSM8K}
& \textbf{MATH-500}
& \textbf{HumanEval}
& \textbf{MBPP} \\
\midrule
LLaDA-1.5$^{\ddagger}$ \citep{zhu2026llada}
& 83.3 & --   & 39.6 & 39.9 \\
LLaDOU$^{\dagger}$ \citep{huang2026reinforcing}
& 88.1 & 41.1 & 59.1 & 51.6 \\
\midrule
d1 \citep{zhao2025d1}
& 81.1 & 38.6 & --   & --   \\
DMPO \citep{zhu2025dmpo}
& 82.4 & 38.2 & -- & -- \\
wd1 \citep{tang2026wd1}
& 80.8 & 34.4 & --   & --   \\
d-TreeRPO \citep{pan2026dtreerpo}
& 81.2 & 37.7 & --   & --   \\
ESPO \citep{ou2026principled}
& 82.3 & 39.0 & 42.1 & 44.6 \\
GDPO \citep{rojas2026improving}
& 82.8 & 39.6 & 39.6 & \textbf{50.6} \\
PAPO \citep{shao2026papo}
& 82.4 & 35.6 & 34.8 & 43.2 \\
StableDRL \citep{zhong2026stabledrl}
& 86.2 & 45.2 & -- & -- \\
SPG \citep{wang2026spg}
& 86.1 & 40.0 & --   & --   \\
JustGRPO \citep{ni2026justgrpo}
& 89.6 & 46.0 & 50.6 & 48.6 \\
\midrule
\textbf{ForkGRPO}
& \textbf{90.4}
& 46.6
& \textbf{51.8}
& 49.0 \\
\textbf{ForkGRPO-Fast}
& 90.3
& \textbf{46.8}
& \textbf{51.8}
& 48.8 \\
\bottomrule
\end{tabular}
\vspace{-3mm}
\end{table}

\textsc{ForkGRPO} improves over JustGRPO on all four benchmarks, achieving 90.4\%, 46.6\%, 51.8\%, and 49.0\% on GSM8K, MATH-500, HumanEval, and MBPP, respectively. These correspond to gains of $0.8$, $0.6$, $1.2$,
and $0.4$ percentage points over JustGRPO.
The improvement is observed on both mathematical reasoning and code
generation, suggesting that conditional branching and finer-grained credit
assignment complement the sequential RL formulation rather than benefiting
only a particular task family.
\textsc{ForkGRPO-fast} preserves nearly all of this performance despite its reduced
training cost.

\subsection{Ablation Studies}
\label{sec:ablation}
All ablations in this section are conducted on GSM8K. Each variant starts from the same public JustGRPO
checkpoint and undergoes 10 additional optimizer updates under each of five training seeds. We report mean accuracy $\pm$ sample standard deviation across the five runs. This shared initialization and fixed update budget enable controlled comparisons of component contributions within our computational budget.
Further details are provided in
Appendix~\ref{app:ablation_protocol}, while the ablation of
\textsc{ForkGRPO-Fast} is reported separately in
Appendix~\ref{sec:fast_ablation}.
\subsubsection{Forking and Credit Assignment}
\label{sec:ablation_components}

We compare four variants in
Table~\ref{tab:component_ablation}.
JustGRPO samples 16 complete responses per prompt
without forking.
All forked variants use Entropy Boundary selection
with $K=3$ fork states and $M=8$ responses per group,
yielding 32 scored completions per prompt,
as in the main experiments.
Their sampling configuration is fixed, and only
the training objective varies.
\textbf{Fork + Group} uses only
$\mathcal L_{\mathrm{group}}$ to train on root
responses and forked suffixes.
\textbf{Fork + Segment} uses continuation rewards
to estimate segment advantages but updates only
the reference response through
$\mathcal L_{\mathrm{seg}}$.
\textbf{ForkGRPO} combines both objectives.

Full \textsc{ForkGRPO} achieves the highest mean accuracy,
outperforming either objective alone.
This supports the complementary roles of
within-state group comparisons and segment-level
credit along the reference trajectory.
\begin{table}[t]
\centering
\fontsize{9pt}{8pt}\selectfont
\setlength{\tabcolsep}{4pt}
\renewcommand{\arraystretch}{1.15}
\caption{
\textbf{Component ablation on GSM8K.}
Forked variants share the same sampling
configuration and differ only in their
training objective.
}
\label{tab:component_ablation}
\begin{tabular}{@{}lccc@{}}
\toprule
\textbf{Method}
& \textbf{Fork}
& \textbf{Credit}
& \textbf{Acc. (\%)} \\
\midrule
JustGRPO
& No
& group
& $88.81 \pm 0.53$ \\

Fork + Group
& Yes
& group
& $89.69 \pm 0.21$ \\

Fork + Segment
& Yes
& segment
& $89.45 \pm 0.29$ \\

\textbf{ForkGRPO}
& Yes
& segment + group
& $\mathbf{90.31} \pm 0.06$ \\
\bottomrule
\end{tabular}
\end{table}

\subsubsection{Fork-State Selection}
\label{sec:ablation_anchor}

We compare four fork-state selection strategies
while fixing the full \textsc{ForkGRPO} objective and sampling
configuration ($K=3$, $M=8$).
Each strategy selects one branching position from
each of the three 32-token windows:
$[32,64)$, $[64,96)$, and $[96,128)$.
\textbf{Fixed} selects the center of each window.
\textbf{Random Boundary} samples uniformly from
the boundary candidates defined in
Section~\ref{sec:method_forking}.
\textbf{Max Entropy} selects the highest-entropy
position within each window without a boundary constraint.
\textbf{Entropy Boundary}, our proposed strategy,
selects the highest-entropy boundary candidate.

\begin{table}[t]
\centering
\fontsize{9pt}{8pt}\selectfont
\setlength{\tabcolsep}{4pt}
\renewcommand{\arraystretch}{1.15}
\caption{
\textbf{Fork-state selection on GSM8K.}
All variants differ only in how branching
positions are selected within each window.
}
\label{tab:anchor_ablation}
\begin{tabular}{@{}lccc@{}}
\toprule
\textbf{Selection}
& \textbf{Boundary}
& \textbf{Entropy}
& \textbf{Acc. (\%)} \\
\midrule
Fixed
& No & No
& $89.10 \pm 0.33$ \\

Random Boundary
& Yes & No
& $89.25 \pm 0.21$ \\

Max Entropy
& No & Yes
& $89.67 \pm 0.16$ \\

Entropy Boundary
& Yes & Yes
& $\mathbf{90.31} \pm 0.06$ \\
\bottomrule
\end{tabular}
\end{table}
Table~\ref{tab:anchor_ablation} shows that
Entropy Boundary achieves the highest mean accuracy,
outperforming both Random Boundary and Max Entropy.
This supports combining predictive uncertainty
with textual boundaries to target uncertain
reasoning transitions, where alternative continuations
may explore different solution paths.

\subsubsection{Credit Density and Sampling Allocation}
\label{sec:ablation_km}
\label{sec:branch_depth_analysis}

\begin{figure}[t]
    \centering
    \includegraphics[width=0.9\columnwidth]{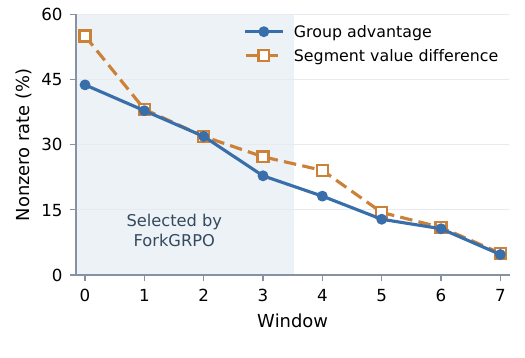}
    \vspace{-6mm}
    \caption{
    Nonzero rates of group advantages and window segment value differences, averaged
    over the first 10 training steps on GSM8K with LLaDA-8B-Instruct.
    Both signals become substantially sparser at later segments.
    Shading indicates the first four 32-token windows covered by the default
    \textsc{ForkGRPO} configuration ($K=3$).
    }
    \label{fig:branch_depth_nonzero}
\end{figure}

\paragraph{Credit density across windows.} We next examine how the usefulness of branching changes across the fixed window defined in Section~\ref{sec:method_forking}. With $W=32$, each candidate fork window spans 32 consecutive response positions.
Figure~\ref{fig:branch_depth_nonzero} reports the fraction of training
examples with nonzero group advantages and nonzero segment value differences
at each window. Both signals become progressively sparser in later
windows, indicating that forks placed later in the trajectory yield fewer
effective policy-gradient signals~\citep{yu2025dapo}.
This trend motivates concentrating the branching budget in the early portion
of the response. Our default configuration ($K=3$) selects one intermediate fork state
from each of the windows $[32,64)$, $[64,96)$, and $[96,128)$.
These three fork states, together with the root state at position $0$,
lie within the first four 32-token windows highlighted in
Figure~\ref{fig:branch_depth_nonzero}, where credit signals remain
relatively dense.

\paragraph{Sampling allocation.}
We next study how performance changes when the same number of responses
is distributed across different numbers of branching states.
Table~\ref{tab:km_ablation} keeps $(K+1)M=32$ scored responses while varying
$K$ and $M$. All candidate windows retain the fixed width $W=32$; for a given
$K$, we use the first $K$ windows following the root. The best performance is obtained with $K=3,M=8$, reaching a mean accuracy of $90.31\%$. Smaller $K$ provides less intermediate-state coverage,
whereas larger $K$ reduces the number of continuations per state and extends
branching into later windows with sparser credit signals. These results support
$K=3,M=8$ as a favorable balance between branching coverage and per-state
estimation.

\begin{table}[t]
\centering
\fontsize{9pt}{8pt}\selectfont
\setlength{\tabcolsep}{12pt}
\renewcommand{\arraystretch}{1.15}

\caption{
\textbf{Sampling allocation on GSM8K.}
}
\vspace{-2mm}
\label{tab:km_ablation}

\begin{tabular}{cccc}
\toprule
\textbf{$K$}
& \textbf{$M$}
& \textbf{Groups}
& \textbf{Acc. (\%)} \\
\midrule
0 & 32 & 1
& $89.14 \pm {\mathrm{0.38}}$ \\

1 & 16 & 2
& $89.36 \pm {\mathrm{0.19}}$ \\

\textbf{3} & \textbf{8} & \textbf{4}
& $\mathbf{90.31} \pm 0.06$ \\

7 & 4 & 8
& $89.93 \pm {\mathrm{0.11}}$ \\
\bottomrule
\end{tabular}
\vspace{-4mm}
\end{table}

\subsection{Sampling Efficiency}
\label{sec:efficiency}

We compare the sampling efficiency of \textsc{ForkGRPO} with the officially
reported JustGRPO-LoRA baseline on GSM8K.
Both methods use the same LoRA configuration,
base model, training data, optimizer settings,
global prompt batch size, response length,
and evaluation protocol to ensure a fair comparison.

JustGRPO-LoRA samples 16 full responses per
prompt and trains for 200 steps.
With $L=256$, each prompt requires
$16L=4096$ newly generated tokens.
\textsc{ForkGRPO} samples 8 initial responses and
8 continuations from each of 3 branching
states, and trains for 100 steps.

Our branching windows give
$a_1\in[32,63]$, $a_2\in[64,95]$, and
$a_3\in[96,127]$.
The number of newly generated tokens per
prompt therefore satisfies
\vspace{-2mm}
\begin{equation}
    N_{\mathrm{tok}}
    =8\left(4L-\sum_{j=1}^{3}a_j\right)
    \in[5912,6656].
    \vspace{-2mm}
\end{equation}
With equal prompt batch sizes, the cumulative
generation budget of \textsc{ForkGRPO}, normalized to
JustGRPO-LoRA, lies in the range
\vspace{-2mm}
\begin{equation}
    \left[
    \frac{100\times5912}{200\times4096},
    \frac{100\times6656}{200\times4096}
    \right]
    \approx [0.7217,0.8125].
\vspace{-1mm}
\end{equation}
Even the upper bound is below the baseline budget,
corresponding to a reduction of 18.75--27.83\%
in cumulative generated tokens.

\begin{table}[t]
\centering
\caption{
\textbf{Sampling efficiency on GSM8K.}
Both methods use LoRA.
Total token budgets are normalized to
JustGRPO-LoRA; ranges denote analytical bounds.
}
\vspace{-2mm}
\label{tab:efficiency}
\small
\renewcommand{\arraystretch}{1.15}
\fontsize{8pt}{7pt}\selectfont
\setlength{\tabcolsep}{4pt}
\begin{tabular*}{\columnwidth}{
@{\extracolsep{\fill}}lcc@{}
}
\toprule
\textbf{Metric}
& \textbf{JustGRPO}
& \textbf{ForkGRPO} \\
\midrule
Training updates
& 200
& 100 \\
Scored responses / prompt
& 16
& 32 \\
New tokens / prompt
& 4096
& 5912--6656 \\
Total token budget
& 1.0000
& $[0.7217,\,0.8125]$ \\
\midrule
GSM8K accuracy (\%)
& 89.6
& \textbf{90.4} \\
\bottomrule
\end{tabular*}
\vspace{-4mm}
\end{table}


\textsc{ForkGRPO} achieves 90.4\% accuracy on GSM8K,
surpassing JustGRPO-LoRA by 0.8 percentage points while using
18.75--27.83\% fewer newly generated tokens.
These results show that conditional branching can provide finer-grained
exploration without requiring a larger cumulative generation budget.

\textsc{ForkGRPO-fast} can further reduce the generation budget through adaptive rollout. Each continuation
group starts with $M_{\mathrm{pilot}}=4$ samples and is expanded to $M=8$ only when the pilot samples exhibit reward variation. Its exact token budget is therefore data dependent: stronger base-model performance or easier tasks may produce more homogeneous groups and require fewer expansions, whereas more challenging tasks may approach the full \textsc{ForkGRPO} budget. Thus, \textsc{ForkGRPO-fast} uses no more generation than \textsc{ForkGRPO} and typically reduces it, while retaining comparable task performance.

\section{Conclusion}
\label{sec:conclusion}

We introduced \textsc{ForkGRPO}, an RL framework that integrates
intermediate-state exploration with fine-grained credit
assignment for masked dLLMs. \textsc{ForkGRPO} branches from high-entropy states along a
reference trajectory, prioritizing textual boundaries
and reusing shared prefixes. The resulting continuations
provide complementary learning signals through within-state
group comparisons and Monte Carlo estimates of segment
value changes. Experiments across four benchmarks
demonstrate improvements in mathematical reasoning and
code generation, while ablations support the contributions
of both learning signals and the proposed branching strategy.
\textsc{ForkGRPO-fast} further reduces rollout and gradient replay
costs, achieving substantial training acceleration on
GSM8K with negligible accuracy loss. These results support
conditional branching and finer-grained credit assignment
as practical tools for improving the accuracy and efficiency
of dLLM post-training.

\section*{Limitations}

The computational cost of RL post-training limits the
scope of our empirical evaluation.
Our main results for \textsc{ForkGRPO} and \textsc{ForkGRPO-fast} are
based on a single training run per method and task,
so their variability across training seeds remains
unmeasured.
We complement these experiments with GSM8K ablations
averaged over 5 training seeds using short continuation
runs from a shared public JustGRPO checkpoint.
This protocol enables controlled comparisons of
individual components within our computational budget.
Extending these comparisons to full training runs
across multiple seeds would provide a more comprehensive
assessment of training stability and component
contributions.





\bibliography{references} 

\newpage
\appendix
\section{Why ForkGRPO for Masked Diffusion Language Models?}
\label{app:why_mdm}

The core mechanism of \textsc{ForkGRPO}, branching from an intermediate state,
sampling multiple conditional continuations, and using their terminal rewards
for finer-grained credit assignment, is not inherently restricted to masked
diffusion language models (MDMs). In principle, an autoregressive (AR) model
could similarly branch from a prefix and estimate the value of that prefix
from multiple sampled suffixes. Our focus on MDMs is therefore not based on
an exclusivity claim, but on a training problem that is particularly natural
and consequential in the MDM setting.

\paragraph{MDMs expose a distinctive exploration problem.}
Unlike AR models, whose generation order is intrinsically left-to-right, MDMs
can resolve masked positions in flexible orders. Recent work shows that this
flexibility can be detrimental to reasoning during RL: uncertain decisions may
be deferred until surrounding tokens constrain their possible outcomes,
reducing opportunities to explore alternative reasoning paths
\citep{ni2026justgrpo}. Sequential unmasking alleviates this issue by forcing
such decisions to be resolved before later context is available. \textsc{ForkGRPO} 
builds directly on this observation. Once an uncertainty-sensitive
intermediate state is exposed along the sequential rollout, \textsc{ForkGRPO} revisits
that state and explicitly samples multiple possible futures rather than relying
only on independent prompt-level trajectories.

\paragraph{Conditional forking is especially natural for partially denoised states.}
At generation position $i$, an MDM state has the form
given in Eq.~\eqref{eq:sequential_state}, which explicitly represents a partially resolved sequence and an unresolved remainder. Forking from this state probes alternative futures under the same committed
content, enabling both within-state comparison and estimation of how the
expected terminal reward changes as additional content becomes fixed.

More generally, this value-based view is not restricted to left-to-right
prefixes. Native MDM states may contain committed tokens at arbitrary
positions, with unresolved regions conditioned on context from the left, the
right, or both. Conditional rollouts can therefore quantify the value change
associated with committing any newly resolved portion of the sequence, rather
than only extending a prefix as in an AR model. While we use sequential
unmasking here to obtain controlled reasoning trajectories, the underlying
credit-assignment view naturally extends to more general MDM decoding orders.

\paragraph{Relation to autoregressive models.}
The core idea of forking is not restricted to MDMs: an AR model
could similarly branch from a prefix, sample multiple suffixes, and estimate
prefix-level value changes. Our focus on MDMs instead follows from the
structure of their generation process. Flexible decoding orders can defer
high-uncertainty decisions, making intermediate-state exploration particularly
relevant, while partially denoised states are not inherently limited to
left-to-right prefixes and may contain committed context at arbitrary
positions. \textsc{ForkGRPO} therefore studies conditional exploration and credit
assignment in this broader state space, using sequential unmasking as a
controlled instance in our experiments. 
We expect \textsc{ForkGRPO} to be plausibly applicable to AR models as well, although
its theoretical motivation is less distinctive in the AR setting, where
generation is already restricted to left-to-right prefix states.

\section{Theoretical Analysis of Segment Credit}
\label{app:mc_estimation}

We analyze the estimation of segment value
changes under the fixed rollout policy.
The comparison concerns raw differences
$\Delta V_j$, before shrinkage, and a
response-level advantage shared across segments.

\subsection{Estimation Guarantee}

\begin{theorem}[Accuracy of segment credit]
\label{thm:segment_mc}
Fix a prompt $q$, a policy $\pi_{\mathrm{old}}$,
and integers $K\geq 1$ and $M\geq 2$.
Assume that the initial $M$ responses are i.i.d.\
under $\pi_{\mathrm{old}}$, with the first response
designated as $y^{\mathrm{ref}}$.
The branching positions $a_1<\cdots<a_K$ are
selected from the reference response alone using
a fixed rule that does not depend on $M$.
Conditional on the initial responses, all
auxiliary continuations are sampled independently
under $\pi_{\mathrm{old}}$ from their selected
states.
Assume rewards lie in $[r_{\min},r_{\max}]$, and
let $\nu^2=(r_{\max}-r_{\min})^2/4$.

For $j=0,\ldots,K$, let $V_j$ be the expected
terminal reward from the state before position
$a_j$ in the reference response, where $a_0=0$.
Let $\widehat V_j$ be the mean reward of group
$j$, including all $M$ initial responses when
$j=0$.
Let $a_{K+1}=L$ and set
$V_{K+1}=\widehat V_{K+1}=R(q,y^{\mathrm{ref}})$.
For each segment
$\mathcal I_j=\{a_j,\ldots,a_{j+1}-1\}$, define
\begin{equation}
\begin{aligned}
    d_j&=V_{j+1}-V_j,\\
    \Delta V_j&=\widehat V_{j+1}-\widehat V_j,
    \quad j=0,\ldots,K.
\end{aligned}
\end{equation}
For every such segment,
\begin{equation}
\begin{aligned}
    \mathbb E[\Delta V_j-d_j\mid q]&=0,\\
    \mathbb E[(\Delta V_j-d_j)^2\mid q]
    &\leq u_j,
\end{aligned}
\label{eq:mc_error}
\end{equation}
where $u_j=2\nu^2/M$ for $j=0,\ldots,K-1$
and $u_K=\nu^2/M$, as defined in Section~\ref{sec:method_segment}.

To compare these estimates with a shared
response-level advantage, define
\begin{equation}
\begin{aligned}
    \bar d&=\frac{1}{K+1}\sum_{j=0}^{K}d_j,\\
    H(q)&=\mathbb E\!\left[
      \frac{1}{K+1}\sum_{j=0}^{K}(d_j-\bar d)^2
      \,\middle|\,q\right].
\end{aligned}
\label{eq:segment_heterogeneity}
\end{equation}
For any scalar advantage $C$ shared across
these segments,
\begin{equation}
    \mathbb E\!\left[
      \frac{1}{K+1}\sum_{j=0}^{K}(C-d_j)^2
      \,\middle|\,q\right]\geq H(q).
    \label{eq:shared_credit_error}
\end{equation}
The advantage $C$ may depend on the reference
and arbitrarily many prompt rollouts.
In contrast,
\begin{equation}
    \mathbb E\!\left[
      \frac{1}{K+1}\sum_{j=0}^{K}
      (\Delta V_j-d_j)^2\,\middle|\,q\right]
    \leq \frac{2\nu^2}{M}.
    \label{eq:segment_average_error}
\end{equation}
Consequently, if $H(q)>0$ and
\begin{equation}
    M>\frac{2\nu^2}{H(q)},
    \label{eq:segment_comparison_condition}
\end{equation}
the raw segment estimates have strictly lower
average mean squared error than any advantage
shared across these segments.
\end{theorem}

All expectations conditioned only on $q$
average over reference, initial-group, and
auxiliary sampling.
The target is the vector of segment value
changes $(d_0,\ldots,d_K)$.
If $H(q)=0$, the estimation bound still holds,
but the strict comparison does not apply.

\subsection{Proof}

\begin{proof}
Let $Y_0=(y^{(0,1)},\ldots,y^{(0,M)})$ denote
the initial responses and define
$e_j=\widehat V_j-V_j$ for $j=0,\ldots,K+1$.
By the terminal-value definition, $e_{K+1}=0$.
Conditioning on $Y_0$ fixes the reference
and all branching states.
For $1\leq j\leq K$, the auxiliary rewards
are conditionally i.i.d., so
\begin{equation}
\begin{aligned}
    \mathbb E[e_j\mid q,Y_0]&=0,\\
    \mathbb E[e_j^2\mid q,Y_0]
    &=\frac{\sigma_j^2}{M}
    \leq\frac{\nu^2}{M},
\end{aligned}
\label{eq:appendix_conditional_mc}
\end{equation}
where $\sigma_j^2$ is the conditional reward
variance.
The inequality follows because a reward
supported on $[r_{\min},r_{\max}]$ has variance
at most $\nu^2$.

At the prompt state, the initial rewards are
i.i.d.\ with variance $\sigma_0^2$. Thus,
\begin{equation}
\begin{aligned}
    \mathbb E[e_0\mid q]&=0,\\
    \mathbb E[e_0^2\mid q]
    &=\frac{\sigma_0^2}{M}
    \leq\frac{\nu^2}{M}.
\end{aligned}
\end{equation}

For $1\leq j<K$, conditional independence of
the auxiliary batches and their conditional
zero means imply
$\mathbb E[e_je_{j+1}\mid q,Y_0]=0$.
For $j=0$, the error $e_0$ is fixed given
$Y_0$, while
$\mathbb E[e_1\mid q,Y_0]=0$.
Taking expectations therefore gives
\begin{equation}
    \mathbb E[e_je_{j+1}\mid q]=0,
    \quad j=0,\ldots,K-1.
\end{equation}
For $j=0,\ldots,K-1$, since
$\Delta V_j-d_j=e_{j+1}-e_j$,
\begin{equation}
\begin{aligned}
    \mathbb E[\Delta V_j-d_j\mid q]&=0,\\
    \mathbb E[(\Delta V_j-d_j)^2\mid q]
    &=\mathbb E[e_{j+1}^2\mid q]
      +\mathbb E[e_j^2\mid q]\\
    &\leq\frac{2\nu^2}{M}.
\end{aligned}
\end{equation}
For the final segment,
$\Delta V_K-d_K=-e_K$, so
\begin{equation}
\begin{aligned}
    \mathbb E[\Delta V_K-d_K\mid q]&=0,\\
    \mathbb E[(\Delta V_K-d_K)^2\mid q]
    &\leq\frac{\nu^2}{M}=u_K.
\end{aligned}
\end{equation}
This proves Eq.~\eqref{eq:mc_error}.
Averaging these bounds over the $K+1$ segments
and using $u_j\leq 2\nu^2/M$ for every $j$
gives Eq.~\eqref{eq:segment_average_error}.

For the shared-advantage comparison, the
following identity holds for every realized
reference and every scalar $C$:
\begin{equation}
\begin{aligned}
    \frac{1}{K+1}\sum_{j=0}^{K}(C-d_j)^2
    &=\frac{1}{K+1}\sum_{j=0}^{K}(d_j-\bar d)^2\\
    &\quad+(C-\bar d)^2.
\end{aligned}
\end{equation}
The cross term vanishes because
$\sum_{j=0}^{K}(d_j-\bar d)=0$.
Dropping the final nonnegative term and
taking expectations proves
Eq.~\eqref{eq:shared_credit_error}, without
requiring independence between $C$ and the
reference.

Because the reference is the first i.i.d.\
response and its branching positions depend
only on it through a rule independent of $M$,
$H(q)$ does not change with $M$.
If $H(q)>0$, condition
\eqref{eq:segment_comparison_condition}
makes the upper bound for the raw segment
estimates strictly smaller than the lower
bound for every shared advantage.
\end{proof}
\subsection{Interpretation and Scope}

The comparison measures how accurately a credit
signal represents each segment's change in
expected terminal reward. When these changes
differ across segments, a shared response-level
advantage cannot represent all of them accurately.
Increasing the number of complete rollouts can
improve estimation of the shared advantage, but
does not resolve this limitation in granularity.
Sampling continuations from intermediate states
allows the corresponding values, and hence
segment-level value changes, to be estimated
separately.

The result concerns estimation error relative
to segment-level value changes. It does not
compare against methods that also construct
segment-specific estimates from complete
rollouts, establish policy-gradient unbiasedness,
or guarantee improved training performance.

Including the reference response in the initial
group does not invalidate the prompt-conditioned
bound. Conditional on a fixed reference, however,
\begin{equation}
    \mathbb E[e_0 \mid q,y^{\mathrm{ref}}]
    = \frac{R(q,y^{\mathrm{ref}})-V_0}{M}.
\end{equation}
The unbiasedness statement for the first segment
in Eq.~\eqref{eq:mc_error} therefore also averages
over reference sampling.
Conditional on the reference and selected fork
states, auxiliary continuations are sampled
independently from their corresponding states.
Selecting fork states using entropy along the
reference thus preserves the unbiasedness of
these auxiliary value estimates; no stopping-time
assumption is required.

The raw segment differences telescope:
\begin{equation}
    \sum_{j=0}^{K}\Delta V_j
    = R(q,y^{\mathrm{ref}})-\widehat V_0.
\end{equation}
The segment advantages after shrinkage need not
satisfy this identity.

\subsection{Noise-Aware Shrinkage}
\label{app:segment_shrinkage}

The scale $u_j$ bounds the prompt-averaged mean
squared estimation error; it is not the exact
conditional variance for an individual reference.
We use this conservative noise scale to attenuate
small estimated value changes:
\begin{equation}
\begin{aligned}
    \alpha_j
    &= \frac{(\Delta V_j)^2}
      {(\Delta V_j)^2+u_j+\epsilon_s},\\
    A_j^{\mathrm{seg}}
    &= \alpha_j \Delta V_j,
\end{aligned}
\end{equation}
where $\epsilon_s>0$.
Since $0\leq\alpha_j<1$,
\begin{equation}
    |A_j^{\mathrm{seg}}|
    \leq |\Delta V_j|.
\end{equation}
Nonzero value changes retain their signs.
When $(\Delta V_j)^2$ is small relative to
$u_j+\epsilon_s$,
\begin{equation}
    A_j^{\mathrm{seg}}
    \approx \frac{(\Delta V_j)^3}
    {u_j+\epsilon_s}.
\end{equation}
When $(\Delta V_j)^2$ is large relative to this
scale, $\alpha_j$ approaches one and the raw
value change is largely preserved.

For a segment with zero true value change,
shrinkage reduces the squared magnitude of
spurious credit relative to the raw estimate
$\Delta V_j$. For nonzero true changes, it can
introduce bias, so no uniform mean squared error
improvement is claimed for $A_j^{\mathrm{seg}}$.
Theorem~\ref{thm:segment_mc} applies to the raw
estimates $\Delta V_j$, rather than the nonlinear
advantages after shrinkage.

Shrinkage affects only segment-level credit.
The within-state group-relative advantages
remain unchanged, and $\lambda$ controls the
contribution of the segment objective to the
combined loss.

\subsection{Distributed Loss Normalization}
\label{app:loss_normalization}

At accumulation microstep $h$, worker $c$
processes one prompt with $(K+1)M$ responses.
For standard \textsc{ForkGRPO}, the number of newly
generated tokens for this prompt is
\begin{equation}
    N_{\mathrm{tok}}^{(h,c)}
    = M\sum_{j=0}^{K}
      \bigl(L-a_j^{(h,c)}\bigr),
\end{equation}
where $a_0^{(h,c)}=0$.
This count includes tokens from responses with
zero group-relative advantage.
Copied prefixes in auxiliary responses are
excluded, while all tokens of the initial
responses are included.

For clarity, we suppress the worker and
microstep indices in the per-prompt objectives
below. Using the clipped surrogate $\ell$
defined in the main text, the group objective is
\begin{equation}
\begin{aligned}
    \mathcal L_{\mathrm{group}}
    = -\frac{1}{N_{\mathrm{tok}}}
      \sum_{j=0}^{K}\sum_{m=1}^{M}
      \sum_{i=a_j}^{L-1}
      \ell\!\left(
        \rho_i^{(j,m)},
        A_{j,m}^{\mathrm{group}}
      \right).
\end{aligned}
\end{equation}
The segment objective is normalized by the
reference length:
\begin{equation}
    \mathcal L_{\mathrm{seg}}
    = -\frac{1}{L}
      \sum_{j=0}^{K}\sum_{i\in\mathcal I_j}
      \ell\!\left(
        \rho_i^{\mathrm{ref}},
        A_j^{\mathrm{seg}}
      \right).
\end{equation}
The reference response contributes to the group
objective through the initial response group
and receives additional segment-level credit
through $\mathcal L_{\mathrm{seg}}$.
The combined per-prompt loss is
\begin{equation}
    \mathcal L_{\mathrm{ForkGRPO}}
    = \mathcal L_{\mathrm{group}}
      + \lambda\mathcal L_{\mathrm{seg}}.
\end{equation}

Gradients are averaged across the $C$ workers
and summed over $H$ accumulation microsteps,
corresponding to
\begin{equation}
    \mathcal L_{\mathrm{step}}
    = \sum_{h=1}^{H}
      \frac{1}{C}\sum_{c=1}^{C}
      \mathcal L_{\mathrm{ForkGRPO}}^{(h,c)}.
    \label{eq:accumulated_loss}
\end{equation}
No additional factor of $1/H$ is applied.
Token normalization is performed separately
for each prompt before worker averaging and
gradient accumulation.

\section{Experimental Details}
\label{app:experimental_details}

\subsection{Benchmarks}
\label{app:benchmarks}

\paragraph{GSM8K.}
We use GSM8K for grade-school mathematical reasoning and follow the standard
train/test split. During RL training, the final answer is extracted from each
generated response and compared against the ground-truth answer. A response
receives terminal reward $+1$ if the extracted answer is correct and $-1$
otherwise. Evaluation reports exact-match accuracy.

\paragraph{MATH-500.}
We train on the MATH training split, comprising 7,500 problems, and evaluate on the 500-problem
MATH-500 test set. Mathematical responses use the same binary terminal reward as GSM8K.
We report exact-match accuracy after answer normalization.

\paragraph{HumanEval and MBPP.}
Following JustGRPO~\citep{ni2026justgrpo}, we use \textsc{AceCode-Hard} for
RL post-training and report pass@1 on HumanEval and MBPP using their standard
unit tests. We adopt the JustGRPO code reward,
\begin{equation}
    R_{\mathrm{code}}
    =
    R_{\mathrm{fmt}}
    +
    \mathbbm{1}[R_{\mathrm{fmt}}=1]R_{\mathrm{exec}},
\end{equation}
where $R_{\mathrm{fmt}}\in\{0,0.5,1\}$ rewards valid code-block formatting
and Python syntax, and $R_{\mathrm{exec}}\in[0,1]$ is the fraction of unit
tests passed. Only syntactically valid programs are executed.

\subsection{Model and Decoding}
\label{app:model_decoding}

The main experiments initialize \textsc{ForkGRPO} and \textsc{ForkGRPO-fast} directly from LLaDA-8B-Instruct~\citep{nie2025llada}. The ablations in Section~\ref{sec:ablation} instead
initialize from a shared public JustGRPO checkpoint, as detailed in Appendix~\ref{app:ablation_protocol}.
Following JustGRPO~\citep{ni2026justgrpo}, RL trajectories reveal response
tokens sequentially from left to right. The underlying masked dLLM retains
bidirectional self-attention over the partially masked sequence.

We use a fixed response length of $L=256$ for both training and evaluation.
Training rollouts are sampled from the fixed rollout policy with temperature
$1$. At evaluation time, we use the standard LLaDA decoding procedure with
block size $32$, low-confidence remasking, $256$ denoising steps, and
temperature $0$ (greedy decoding). The same evaluation configuration is used
for all models in our experiments.

\paragraph{EOS and padding.}
Generation does not stop at EOS, and no EOS-based
loss mask is applied during training.
Prompt tokens are excluded from the loss, and no
response padding is added.

\subsection{ForkGRPO Configuration}
\label{app:fork_config}

\paragraph{ForkGRPO.}
The default configuration uses $K=3$ intermediate fork states and
$M=8$ continuations per group. Together with the root group, each prompt
contains four continuation groups and 32 terminally scored completions.

For $L=256$, the three fork windows are
\[
    [32,64),\qquad [64,96),\qquad [96,128).
\]
Within each window, we first identify candidate positions immediately
following a newline or sentence-ending punctuation mark. If no such position
exists, all positions in the window are used as candidates. We then select
the candidate with the highest predictive entropy,
\begin{equation}
    a_j
    =
    \arg\max_{i\in\mathcal P_j}
    H_i^{\mathrm{ref}},
\end{equation}
where $\mathcal P_j$ denotes the candidate set in window $j$.
The selected state is immediately before position $a_j$, so that the selected
token itself is resampled by all continuations from that fork. Forked
continuations reuse the committed reference prefix, sample only the remaining
suffix, and do not recursively fork. We set the segment-loss coefficient to
$\lambda=0.25$ throughout the main experiments.

\paragraph{Boundary delimiters.}
We identify textual boundaries using newlines
(\texttt{\textbackslash n}), the punctuation marks
\texttt{.}, \texttt{!}, \texttt{?}, and \texttt{;},
the ideographic full stop (U+3002), and the
fullwidth exclamation mark (U+FF01), question mark
(U+FF1F), and semicolon (U+FF1B).
Periods between two digits are excluded.
After a delimiter, we skip consecutive delimiters
and intervening whitespace to locate the next
token containing non-whitespace text.
Only cuts aligned with token boundaries are eligible:
delimiters remain in the fixed prefix, and each
continuation resamples the next content token.
Random Boundary and Entropy Boundary use the same
delimiter set and eligibility rules
(Section~\ref{sec:ablation_anchor}).

\paragraph{ForkGRPO-Fast.}
The fast variant retains the full group size $M=8$ but first samples
$M_{\mathrm{pilot}}=4$ pilot continuations at each root or fork state.
If every pilot group is homogeneous and all groups share the same observed
reward, rollout collection for that prompt terminates after the pilot stage.
Otherwise, only groups exhibiting reward variation are expanded to $M=8$,
while homogeneous groups remain at $M_{\mathrm{pilot}}=4$.
All value estimates and advantages are computed from the samples actually
collected at each state.

For binary-reward mathematical tasks, we further use \emph{1+1} sparse branch
replay, retaining one representative from each observed reward class with
multiplicity correction. Homogeneous groups require no group-relative replay.
For non-binary code rewards, all collected branches are retained. For both
mathematical reasoning and code generation, we apply random state subsampling
with retention ratio $\rho=0.25$ and inverse-retention weighting.

Adaptive rollout reduces continuation sampling, while sparse branch replay and
state subsampling reduce gradient replay. Their estimator properties are
analyzed in Appendix~\ref{app:fast_analysis}.

\subsection{Optimization}
\label{app:optimization}
Our implementation builds on the official JustGRPO codebase~\citep{ni2026justgrpo}, which is distributed
under the MIT License. Our code release will retain its copyright notice and license text.

All training runs use 8 NVIDIA H200 GPUs, with one prompt per GPU per microbatch and gradient
accumulation over 8 microbatches. Each training step consists of one optimizer update with a global batch of 64 prompts. Under the default \textsc{ForkGRPO} configuration, each prompt yields 32 scored completions:
8 initial responses and 8 continuations from each of the 3 fork states. The main experiments use 100 optimizer updates with a learning rate of $5 \times 10^{-5}$. The ablations in Section~\ref{sec:ablation} use
10 additional optimizer updates from the JustGRPO checkpoint.

We set the policy-ratio clipping threshold to
$\epsilon=0.2$, the group-advantage numerical
stabilizer to $\epsilon_A=10^{-4}$, and the
shrinkage numerical stabilizer to
$\epsilon_s=10^{-12}$.

We use AdamW~\citep{loshchilov2017decoupled} with
$\beta_1=0.9$, $\beta_2=0.999$,
$\epsilon_{\mathrm{Adam}}=10^{-8}$,
no weight decay, and gradient clipping at $1.0$.
The model is adapted using LoRA~\citep{hu2021lora}
with rank $128$, scaling parameter $\alpha=64$,
and dropout $0.05$.

For each optimizer update, the rollout policy $\pi_{\mathrm{old}}$ is held
fixed while the root and forked continuation groups are sampled and terminal
rewards are computed. Monte Carlo state values, segment-level credit,
and within-state group-relative advantages are then constructed from these rewards. The token
log-probabilities required by the training objective are recomputed under the
trainable policy $\pi_\theta$, followed by backpropagation and a single
optimizer update.

Gradients are accumulated across 8 prompt microbatches. Rollout generation is
performed under \texttt{torch.no\_grad()} and does not retain the sampling
computation graph. Intermediate rollout data store only token IDs, fork
locations, terminal rewards, and the statistics required to construct the
training objective.

\subsection{Controlled Continuation Protocol for Ablations}
\label{app:ablation_protocol}

\paragraph{Initialization.}
We initialize the ablations in
Section~\ref{sec:ablation} from the JustGRPO
checkpoint for GSM8K publicly released by its
authors on Hugging Face.\footnote{
\url{https://huggingface.co/nzl-thu}}
All variants and training seeds start from the same
checkpoint, which is reloaded independently before
each continuation run.
\paragraph{Training and evaluation.}
Each variant is trained for 10 additional optimizer
updates with each of 5 distinct random seeds.
We evaluate the checkpoint after the tenth update
and report the mean accuracy and sample standard deviation across the 5 runs.
The training data, reward definition, and evaluation
procedure follow Appendices~\ref{app:benchmarks}
and~\ref{app:model_decoding}.
Unless explicitly varied by the ablation, optimization
hyperparameters follow Appendix~\ref{app:optimization}.

\paragraph{Scope.}
This protocol compares the incremental effects of
individual components from a common trained policy
under a fixed update budget.
The main results in Table~\ref{tab:main_results} are
obtained separately through training initialized
directly from LLaDA-8B-Instruct.
\section{Analysis of Segment-Level Credit Assignment}
\label{app:credit_analysis}

\begin{figure}[t]
    \centering
    \includegraphics[width=\linewidth]{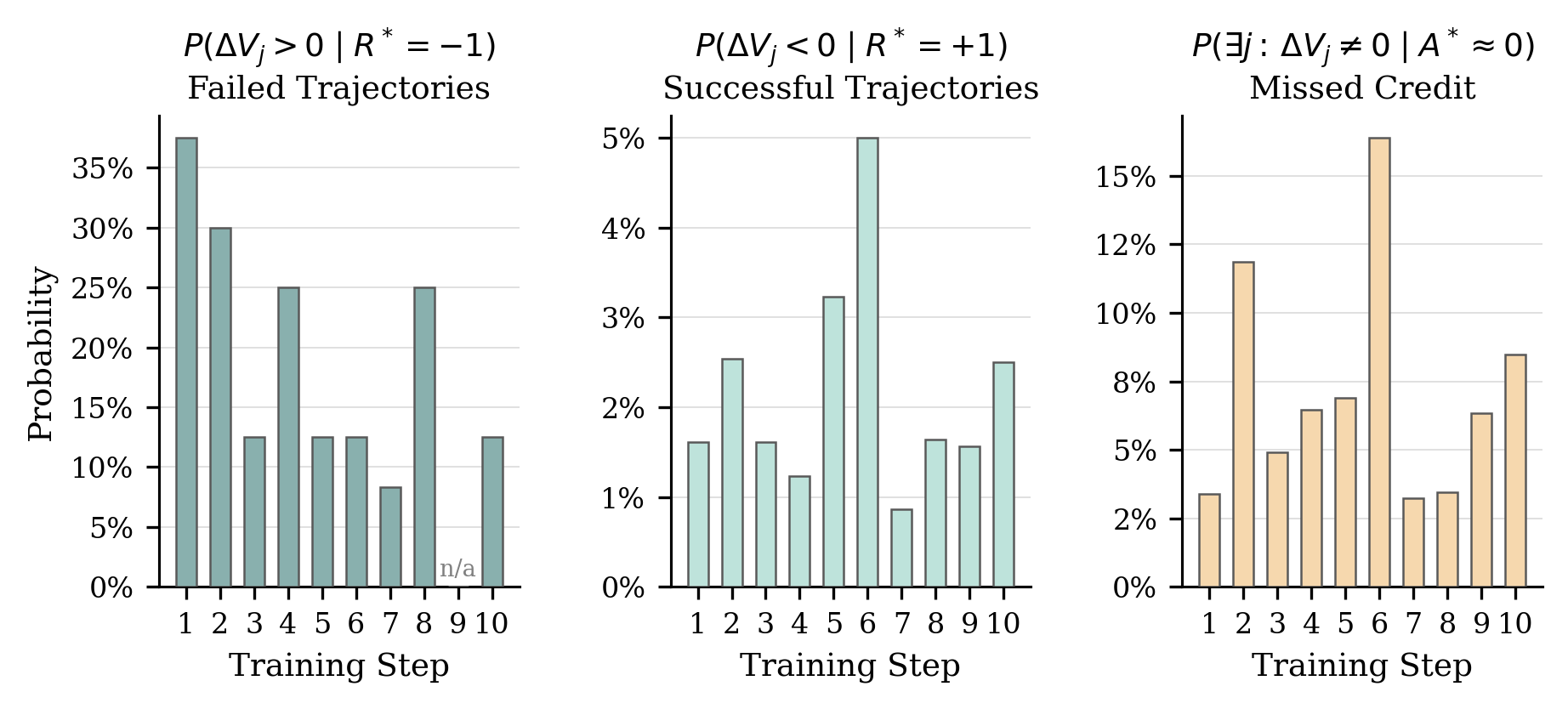}
    \caption{
    \textbf{Diagnosis of segment-level credit over the first 10 training
    steps on GSM8K.}
    \textbf{Left:}
    $P(\Delta V_j>0 \mid R(q,y^{\mathrm{ref}})=-1)$, measuring segments that
    improve continuation value despite an unsuccessful final response.
    \textbf{Middle:}
    $P(\Delta V_j<0 \mid R(q,y^{\mathrm{ref}})=+1)$, measuring segments that
    reduce continuation value despite a successful final response.
    \textbf{Right:}
    $P(\exists j:\Delta V_j\neq0 \mid
    A_{\mathrm{ref}}^{\mathrm{traj}}\approx0)$, measuring segment-level
    signals missed by trajectory-level credit.
    }
    \label{fig:credit_diagnosis}
\end{figure}

We examine whether a single trajectory-level advantage faithfully represents
the contributions of different parts of a response. Using the conditional
continuations already collected by \textsc{ForkGRPO}, we estimate the value of each
root or fork state as
\begin{equation}
    \widehat V_j
    =
    \frac{1}{M}
    \sum_{m=1}^{M} R_{j,m},
    \qquad j=0,\ldots,K,
    \label{eq:app_mc_value}
\end{equation}
and define
\begin{equation}
    \widehat V_{K+1}=R(q,y^{\mathrm{ref}}).
\end{equation}
The value change assigned to segment $j$ is therefore
\begin{equation}
    \Delta V_j
    =
    \widehat V_{j+1}-\widehat V_j,
    \qquad j=0,\ldots,K.
    \label{eq:app_delta_v}
\end{equation}

\paragraph{Misassigned credit.}
A trajectory-level advantage assigns the same outcome-dependent credit to all
tokens in a response, although individual segments can have different effects
on the expected terminal reward. We diagnose two forms of sign mismatch:
\begin{equation}
    P\!\left(
        \Delta V_j>0
        \mid
        R(q,y^{\mathrm{ref}})=-1
    \right)
\end{equation}
and
\begin{equation}
    P\!\left(
        \Delta V_j<0
        \mid
        R(q,y^{\mathrm{ref}})=+1
    \right).
\end{equation}

As shown in Fig.~\ref{fig:credit_diagnosis}, approximately $8\%$--$38\%$ of
segments within failed trajectories nevertheless have positive value change.
Thus, a uniformly negative trajectory-level signal can suppress segments that
improve the prospect of eventual success. The reverse case is less frequent
but remains present: approximately $1\%$--$5\%$ of segments within successful
trajectories have negative value change and would otherwise receive positive
trajectory-level credit.

\paragraph{Missed credit.}
Trajectory-level group-relative credit can also vanish even when meaningful
within-trajectory changes remain. We measure
\begin{equation}
    P\!\left(
        \exists j:\Delta V_j\neq0
        \mid
        A_{\mathrm{ref}}^{\mathrm{traj}}\approx0
    \right),
    \label{eq:app_missed_credit}
\end{equation}
where $A_{\mathrm{ref}}^{\mathrm{traj}}$ denotes the response-level advantage
assigned to the reference trajectory. Figure~\ref{fig:credit_diagnosis} shows
that approximately $3\%$--$17\%$ of such trajectories still contain at least
one segment with nonzero value change. These cases receive little or no
trajectory-level update despite containing identifiable intermediate credit.

\paragraph{Implication.}
The diagnostic reveals two distinct limitations of trajectory-level credit:
\begin{equation}
\begin{aligned}
\text{misassigned credit:}\quad&
\operatorname{sign}(\Delta V_j)
\neq
\operatorname{sign}(A_{\mathrm{ref}}^{\mathrm{traj}}),\\
\text{missed credit:}\quad&
A_{\mathrm{ref}}^{\mathrm{traj}}\approx0,
\quad
\exists j:\Delta V_j\neq0.
\end{aligned}
\end{equation}
\textsc{ForkGRPO} instead uses changes in continuation value to assign credit at the
segment level, recovering learning signals that can be incorrectly signed or
absent under a single response-level advantage.

\subsection{Reliability of Monte Carlo State Values}
\label{sec:value_samples}

Segment credit depends on differences between Monte Carlo estimates of
successive state values. We therefore examine how many continuations are
needed for these estimates to become stable on GSM8K. For a held-out set of
intermediate states, we sample $M_{\max}=32$ continuations and use the
resulting estimates as a high-sample reference. For each
$M\in\{2,4,8,16\}$, we repeat the subsampling procedure 20 times and report
the average value-estimation error and sign agreement of adjacent value
differences.
Specifically, we measure
\begin{equation}
    \operatorname{MAE}
    \left(
        \widehat V^{(M)},
        \widehat V^{(M_{\max})}
    \right)
 \end{equation}   
and
\begin{equation}
    \Pr\!\left[
    \operatorname{sign}
    \left(
        \Delta V_j^{(M)}
    \right)
    =
    \operatorname{sign}
    \left(
        \Delta V_j^{(M_{\max})} 
    \right)
    \right].
    \label{eq:value_sign_agreement}
\end{equation}
Table~\ref{tab:value_samples} shows that Monte Carlo state-value estimates
stabilize rapidly as $M$ increases. From $M=2$ to $M=8$, value MAE decreases
from $0.29$ to $0.07$, while $\Delta V$ sign agreement rises from $74.4\%$
to $92.3\%$. Increasing to $M=16$ provides only modest additional gains,
supporting $M=8$ as a favorable trade-off between reliable segment credit and
branching coverage. Sign agreement is particularly important because it
determines the direction of the segment-level update.

\begin{table}[t]
\centering
\fontsize{8pt}{7pt}\selectfont
\setlength{\tabcolsep}{15pt}
\renewcommand{\arraystretch}{1.10}

\captionof{table}{
\textbf{Monte Carlo value-estimation reliability.}
}
\vspace{-2mm}
\label{tab:value_samples}

\begin{tabular}{ccc}
\toprule
\textbf{$M$}
& \textbf{Value MAE $\downarrow$}
& \textbf{$\Delta V$ Agree. $\uparrow$} \\
\midrule
2  & 0.29 & 74.4 \\
4  & 0.22 & 84.1 \\
8  & 0.07 & 92.3 \\
16 & 0.06 & 94.8 \\
\bottomrule
\end{tabular}
\end{table}

\section{Analysis of ForkGRPO-Fast}
\label{app:fast_analysis}

We analyze the three efficiency mechanisms in
\textsc{ForkGRPO-Fast}. Random branch and state replay sparsify gradient
computation after rollouts have been collected, whereas adaptive rollout
reduces the number of continuations collected in the first place.

\subsection{Sparse Gradient Replay}
\label{app:fast_unbiased}

We first condition on a fixed set of collected rollouts. Under this
conditioning, all rewards, group-relative advantages, and normalization terms
are fixed.

\paragraph{Branch replay.}
For binary-reward tasks, responses with the same reward within a group share
the same standardized group-relative advantage. Let
$\mathcal G_j^c$ denote the branches in reward class
$c\in\{+,-\}$, with $k_j^c=|\mathcal G_j^c|$, and let
$g_{j,m}$ denote the complete gradient contribution of branch $m$ to the
group objective.

Suppose we uniformly sample
$r_j^c\leq k_j^c$ branches without replacement, yielding
$\widetilde{\mathcal G}_j^c$, and apply multiplicity correction:
\begin{equation}
    \widehat G_j^c
    =
    \frac{k_j^c}{r_j^c}
    \sum_{m\in\widetilde{\mathcal G}_j^c}
    g_{j,m}.
    \label{eq:branch_replay_estimator}
\end{equation}

\begin{proposition}[Unbiased branch replay]
\label{prop:branch_replay}
Conditioned on the collected rollout set,
\begin{equation}
    \mathbb E[\widehat G_j^c]
    =
    \sum_{m\in\mathcal G_j^c}g_{j,m}.
\end{equation}
\end{proposition}

\begin{proof}
Each branch in $\mathcal G_j^c$ is included with probability
$r_j^c/k_j^c$. Therefore,
\begin{equation}
\begin{aligned}
    \mathbb E[\widehat G_j^c]
    &=
    \frac{k_j^c}{r_j^c}
    \sum_{m\in\mathcal G_j^c}
    \frac{r_j^c}{k_j^c}g_{j,m} \\
    &=
    \sum_{m\in\mathcal G_j^c}g_{j,m}.
\end{aligned}
\end{equation}
\end{proof}

Summing over the two reward classes therefore preserves the complete group
gradient in expectation. A homogeneous binary-reward group has zero
group-relative advantage and requires no group replay. This optimization is
used only for binary-reward tasks.

\paragraph{State subsampling.}
Consider a retained branch with $n$ replay states and let $g_i$ denote the
gradient contribution of state $i$. The full contribution is
\begin{equation}
    G=\sum_{i=1}^{n}g_i.
\end{equation}
If $r$ states are sampled uniformly without replacement, we use
\begin{equation}
    \widehat G
    =
    \frac{n}{r}
    \sum_{i\in\widetilde{\mathcal S}}g_i.
    \label{eq:state_replay_estimator}
\end{equation}

\begin{proposition}[Unbiased random state replay]
\label{prop:state_replay}
Uniform state subsampling with inverse-retention correction preserves the
full state-level gradient in expectation:
\begin{equation}
    \mathbb E[\widehat G]=G.
\end{equation}
\end{proposition}

\begin{proof}
Each state is included with probability $r/n$, giving
\begin{equation}
    \mathbb E[\widehat G]
    =
    \frac{n}{r}
    \sum_{i=1}^{n}
    \frac{r}{n}g_i
    =
    \sum_{i=1}^{n}g_i.
\end{equation}

The same argument applies independently within each stratum under stratified
uniform sampling.
\end{proof}
\begin{corollary}[Unbiased sparse replay]
\label{cor:sparse_replay}
Conditioned on the collected rollout set, multiplicity-corrected random branch
replay followed by inverse-retention random state subsampling is an unbiased
estimator of the full \textsc{ForkGRPO} group gradient.
\end{corollary}

This follows by applying conditional expectation first over state sampling
and then over branch sampling. Thus, \emph{1+1} or \emph{2+2} branch replay
can be combined with random state subsampling without changing the expected
gradient update.


\subsection{Adaptive Rollout}
\label{app:fast_adaptive}

Adaptive rollout changes the sampling stage rather than the gradient replay
stage. For each root or fork state, \textsc{ForkGRPO-fast} first collects
$M_{\mathrm{pilot}}$ continuations. If every pilot group is homogeneous and
all groups share the same observed reward $c$, then under the pilot samples
\begin{equation}
    A_{j,m}^{\mathrm{group}}=0,
    \qquad
    \widehat V_j=c,
    \qquad
    \Delta V_j=0
\end{equation}
for all groups and segments. The pilot samples therefore contain neither
within-state nor segment-level credit, and rollout collection for that prompt
is terminated.

Otherwise, only groups exhibiting reward variation are expanded to the full
budget $M$; homogeneous groups remain at $M_{\mathrm{pilot}}$. All subsequent
statistics are computed using the number of continuations actually collected
at each state.

Unlike sparse replay, this procedure is not generally unbiased with respect
to fixed-$M$ \textsc{ForkGRPO} because the number of collected samples depends on the
observed pilot rewards. We therefore treat adaptive rollout as an efficiency
approximation. In contrast, random branch and state replay remain unbiased
conditioned on whatever rollout set adaptive sampling produces.

\section{Ablation of ForkGRPO-Fast}
\label{sec:fast_ablation}

We study the efficiency--performance trade-off of
\textsc{ForkGRPO-Fast} on GSM8K. All variants use the same
LLaDA-8B-Instruct initialization, training data, LoRA configuration,
optimizer, prompt batch size, response length, training schedule, hardware,
and evaluation protocol. Each ablation enables only the efficiency mechanism
indicated in Table~\ref{tab:fast_ablation}, while keeping the remaining
\textsc{ForkGRPO} configuration unchanged.

We measure wall-clock time at the granularity of one complete training step
and report the average over five measured steps for each configuration.
For each step, we separately record rollout generation time, gradient-replay
time, and total step time. Because rollout collection is stochastic, the
relative time spent in rollout and replay can vary across configurations and
training steps. We therefore measure the overall speedup directly from total
step time rather than deriving it from a fixed stage-wise time ratio.
For reference, full \textsc{ForkGRPO} spends on average $59\%$ of its measured
wall-clock time on rollout generation and $41\%$ on gradient replay.

\begin{table}[t]
\centering
\small
\fontsize{8pt}{9pt}\selectfont
\setlength{\tabcolsep}{3pt}
\caption{
\textbf{Ablation of \textsc{ForkGRPO-Fast} on GSM8K.}
Wall-clock measurements are averaged over five training steps.
Rollout, replay, and overall speedups are measured directly relative to
full \textsc{ForkGRPO}; overall speedup is computed from total step time.
}
\label{tab:fast_ablation}
\begin{tabular}{lcccc}
\toprule
\textbf{Variant}
& \textbf{Acc.}
& \textbf{Rollout}
& \textbf{Replay}
& \textbf{Overall} \\
\midrule
\textsc{ForkGRPO} 
& \textbf{90.4}
& $1.00\times$
& $1.00\times$
& $1.00\times$ \\

+ Adaptive rollout
& 90.3
& $1.76\times$
& $1.00\times$
& $1.37\times$ \\

+ Branch replay (1+1)
& 90.3
& $1.00\times$
& $2.26\times$
& $1.24\times$ \\

+ State subsampling ($\rho{=}0.25$)
& 90.2
& $1.00\times$
& $2.43\times$
& $1.34\times$ \\

\textsc{ForkGRPO-Fast}
& 90.3
& $\mathbf{1.76}\times$
& $\mathbf{5.32}\times$
& $\mathbf{2.49}\times$ \\
\bottomrule
\end{tabular}
\end{table}

\paragraph{Adaptive rollout.}
We first enable adaptive rollout alone while leaving gradient replay
unchanged. Each continuation group starts with
$M_{\mathrm{pilot}}=4$ samples and expands to the full $M=8$ only when the
pilot rewards exhibit variation. This reduces rollout time by
$1.76\times$ and yields a directly measured $1.37\times$ overall speedup,
while accuracy decreases by only $0.1$ points
($90.4\%\rightarrow90.3\%$). The result indicates that many continuation
groups do not require the full rollout budget to provide useful credit
signals.

\paragraph{Sparse gradient replay.}
We next keep the rollout procedure unchanged and sparsify only gradient
replay. On GSM8K, \emph{1+1} branch replay retains one representative from
each observed binary reward class within a mixed group. It accelerates replay
by $2.26\times$ and the complete training step by $1.24\times$, while
retaining $90.3\%$ accuracy. Random state subsampling with
$\rho=0.25$ provides a $2.43\times$ replay speedup and a $1.34\times$
overall speedup, with $90.2\%$ accuracy. These results show that dense
gradient replay over every sampled branch and generation state is not
necessary for effective optimization.

\paragraph{Combined acceleration.}
Finally, \textsc{ForkGRPO-Fast} combines adaptive rollout,
\emph{1+1} branch replay, and random state subsampling with
$\rho=0.25$. The combined configuration achieves a $1.76\times$ rollout
speedup and a $5.32\times$ replay speedup, resulting in a directly measured
$2.49\times$ overall training speedup. Despite this substantial reduction in
wall-clock cost, accuracy remains $90.3\%$, only $0.1$ points below full
\textsc{ForkGRPO}. This result supports the central design of \textsc{ForkGRPO-Fast}: the
sampling needed to identify informative credit can be decoupled from the
amount of gradient computation required to learn from it.

\section{Generation Examples}
\label{sec:visualization}

\begin{figure}[t!]
    \centering
    \includegraphics[
        trim={210 145 208 110},
        clip,
        width=\linewidth
    ]{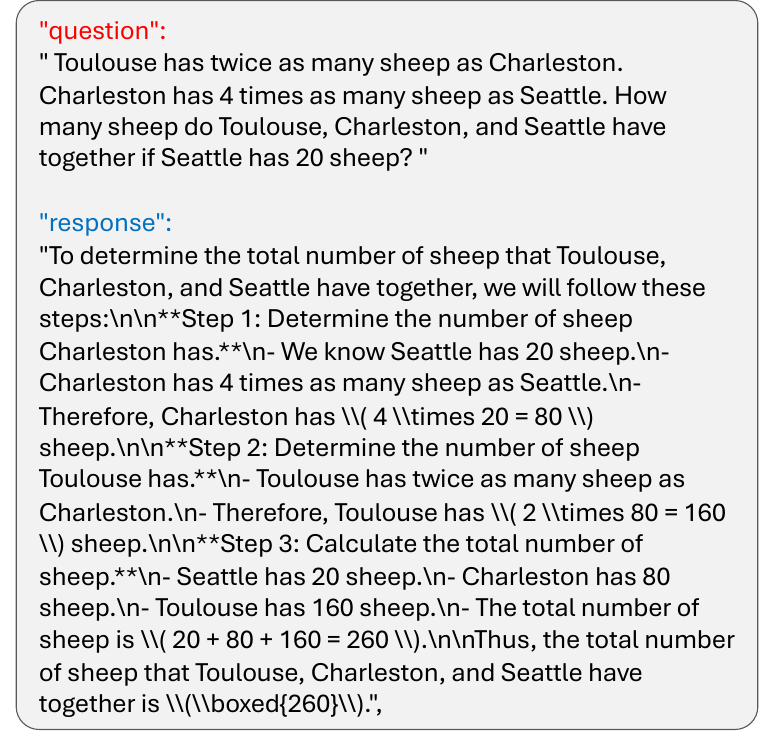}
    \caption{
    \textbf{Representative ForkGRPO generation on GSM8K.}
    The model correctly propagates the multiplicative relations among the
    three quantities, obtaining 80 sheep for Charleston and 160 for Toulouse,
    before producing the correct total of $260$.
    }
    \label{fig:gsm8k_vis1}
\end{figure}

\begin{figure}[t!]
    \centering
    \includegraphics[
        trim={210 254 208 110},
        clip,
        width=\linewidth
    ]{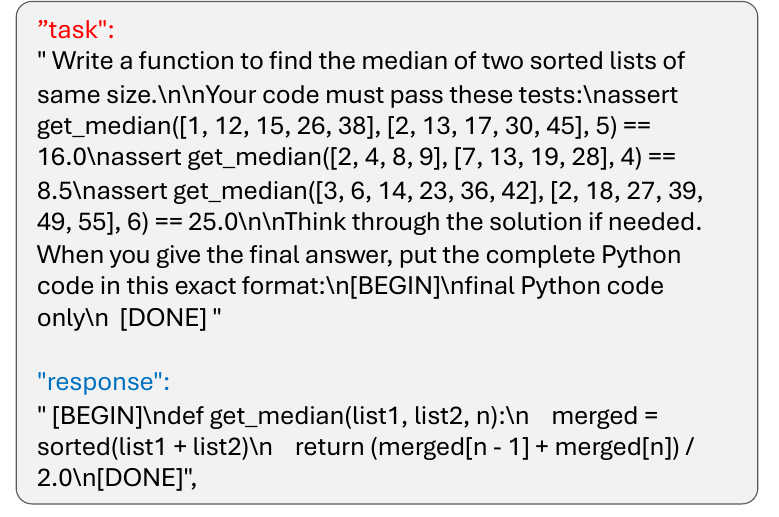}
    \caption{
    \textbf{Representative ForkGRPO generation on MBPP.}
    The model constructs a valid implementation by merging the two input
    lists, sorting the combined sequence, and averaging its two central
    elements. The resulting function satisfies all provided unit tests.
    }
    \label{fig:mbpp_vis2}
\end{figure}

Figures~\ref{fig:gsm8k_vis1} and~\ref{fig:mbpp_vis2} show representative
generations from \textsc{ForkGRPO} on mathematical reasoning and code
generation. In the GSM8K example, the model maintains the required numerical
dependencies across multiple reasoning steps and arrives at the correct final
answer. In the MBPP example, it translates the task specification into a
compact executable program that correctly computes the median of two
equal-sized sorted lists and passes all provided tests. Together, these
examples qualitatively illustrate the coherent reasoning and task completion
behavior learned through RL post-training, complementing the aggregate
benchmark results reported above.

\section{ForkGRPO Pseudocode}
\label{app:pseudocode}

\begin{algorithm}[t]
\caption{\textsc{ForkGRPO} }
\label{alg:forkgrpo}
\small
\raggedright
\begin{algorithmic}[1]
\Require Prompt batch $\mathcal B$,
policies $\pi_{\mathrm{old}},\pi_\theta$,
fork windows $\{\mathcal W_j\}_{j=1}^{K}$,
group size $M$, length $L$, coefficient $\lambda$

\ForAll{$q\in\mathcal B$}
    \State Sample $M$ initial responses;
    use the first as $y^{\mathrm{ref}}$
    \State Record pre-sampling entropies
    $\{H_t^{\mathrm{ref}}\}$
    \State Set $a_0=0$ and $a_{K+1}=L$

    \For{$j=1,\ldots,K$}
        \State Construct boundary-preferred
        candidates $\mathcal P_j$ in $\mathcal W_j$
        \State Select
        $a_j=\arg\max_{t\in\mathcal P_j}
        H_t^{\mathrm{ref}}$
        \State Sample $M$ continuations from
        $s_{a_j}(q,y^{\mathrm{ref}})$
    \EndFor

    \State Evaluate terminal rewards $R_{j,m}$
    \State Estimate $\widehat V_0,\ldots,\widehat V_K$
    from group means
    \State Set
    $\widehat V_{K+1}=R(q,y^{\mathrm{ref}})$
    \State Compute $\Delta V_j$ and shrunk
    advantages $A_j^{\mathrm{seg}}$
    \State Compute within-state advantages
    $A_{j,m}^{\mathrm{group}}$
    \State Recompute token log probabilities
    under $\pi_\theta$
    \State Accumulate gradients of
    $\mathcal L_{\mathrm{group}}
    +\lambda\mathcal L_{\mathrm{seg}}$
\EndFor

\State Update $\theta$ using accumulated gradients
\end{algorithmic}
\end{algorithm}

All responses are sampled under the fixed rollout
policy $\pi_{\mathrm{old}}$ without gradient tracking,
and their token log probabilities are stored.
The reference is selected independently of reward.
Forked continuations resample from position $a_j$,
and only newly generated tokens contribute to
their group objective.
Value estimates and objectives follow
Sections~\ref{sec:method_segment}
and~\ref{sec:method_objective};
gradient accumulation follows
Appendix~\ref{app:loss_normalization}.

\begin{algorithm}[t]
\caption{\textsc{ForkGRPO-Fast}}
\label{alg:forkgrpo_fast}
\small
\raggedright
\begin{algorithmic}[1]
\Require Inputs of Algorithm~\ref{alg:forkgrpo},
pilot size $M_{\mathrm{pilot}}\leq M$,
replay mode $\mathcal R$, retention ratio $\rho$

\ForAll{$q\in\mathcal B$}
    \State Sample $M_{\mathrm{pilot}}$ initial
    responses; use the first as $y^{\mathrm{ref}}$
    \State Select fork states as in
    Algorithm~\ref{alg:forkgrpo}
    \State Sample $M_{\mathrm{pilot}}$ responses
    per fork and score all responses

    \If{all pilot rewards across groups are identical}
        \State \textbf{continue}
    \EndIf

    \State Expand and score mixed-reward groups
    to size $M$
    \State Compute $\widehat V_j$,
    $A_j^{\mathrm{seg}}$, and
    $A_{j,m}^{\mathrm{group}}$
    \State Fix loss normalizers using all samples

    \State Select replay branches using $\mathcal R$
    \State Randomly retain a fraction $\rho$
    of their generated states
    \State Compute weighted
    $\widehat{\mathcal L}_{\mathrm{group}}$
    \State Compute $\mathcal L_{\mathrm{seg}}$
    on $y^{\mathrm{ref}}$ separately
    \State Accumulate gradients of
    $\widehat{\mathcal L}_{\mathrm{group}}
    +\lambda\mathcal L_{\mathrm{seg}}$
\EndFor

\State Update $\theta$ using accumulated gradients
\end{algorithmic}
\end{algorithm}

\textsc{ForkGRPO-Fast} expands only groups with
differing pilot rewards.
All collected responses contribute to value and
advantage estimation using their actual group
sizes $M_j$.
The group-loss denominator is
$N_{\mathrm{tok}}=\sum_{j=0}^{K}M_j(L-a_j)$
and remains fixed during sparse replay.

For binary rewards, $\mathcal R=\textsc{1+1}$
uniformly retains one response per reward class
and weights it by the class size.
$\mathcal R=\textsc{Full}$ retains all responses
and is used for nonbinary rewards.
Random state subsampling uses inverse-inclusion
weights, preserving the group-objective gradient
in expectation conditional on the collected samples.
The reference contributes to the segment objective
regardless of branch selection.

Setting $M_{\mathrm{pilot}}=M$,
$\mathcal R=\textsc{Full}$, and $\rho=1$
recovers standard \textsc{ForkGRPO}.
Prompts skipped because all rewards are identical
have zero group and segment advantages.

\end{document}